\documentclass[twoside]{article}

\usepackage[accepted]{aistats2026}

\usepackage{amsmath}
\usepackage{algorithm}
\usepackage{algpseudocode}
\usepackage{cleveref}
\usepackage{graphicx}
\usepackage{subcaption}
\usepackage{amsthm}
\usepackage{natbib}
\usepackage{xcolor}
\usepackage{amsfonts}
\usepackage{float}
\usepackage{subcaption}
\usepackage{tikz}
\usetikzlibrary{positioning, calc}
\usepackage{etoolbox}

\newcommand{\cc}[1]{\mathcal{#1}}
\newcommand{\bb}[1]{\mathbb{#1}}
\newcommand{\ol}[1]{\overline{#1}}

\newcommand{\mbf}[1]{\mathbf{#1}}

\newtheorem{theorem}{Theorem}

\newtheorem{lemma}{Lemma}
\newtheorem{assumption}{Assumption}
\crefname{assumption}{Assumption}{Assumptions}

\tikzset{
    m_node/.style={circle, draw, thick, minimum size=1cm},
    red_node/.style={m_node, fill=red!40},
    white_node/.style={m_node, fill=white}
}

\newcommand{\nodenetwork}[2][]{%
    \begin{tikzpicture}[node distance=1.5cm]
        \node[#2] (nodeA) {};
        \ifx\relax#1\relax
            \node[red_node, right=of nodeA, yshift=1cm] (nodeB) {Jon};
        \else
            \ifstrequal{#1}{nolabel}{
                \node[red_node, right=of nodeA, yshift=1cm] (nodeB) {};
            }{
                \node[red_node, right=of nodeA, yshift=1cm] (nodeB) {Jon};
            }
        \fi
        \node[white_node, right=of nodeA, yshift=-1cm] (nodeC) {};
        \draw[thick] (nodeA) -- (nodeB);
        \draw[thick] (nodeB) -- (nodeC);
    \end{tikzpicture}%
}

\begin{document}

\twocolumn[

\aistatstitle{Scalable Policy Maximization Under Network Interference}

\aistatsauthor{ Aidan Gleich \And Eric Laber \And  Alexander Volfovsky }

\aistatsaddress{Duke University \And Duke University \And Duke University} ]

\begin{abstract}
  Many interventions, such as vaccines in clinical trials or coupons in online marketplaces, must be assigned sequentially without full knowledge of their effects. Multi-armed bandit algorithms have proven successful in such settings. However, standard independence assumptions fail when the treatment status of one individual impacts the outcomes of others, a phenomenon known as interference. We study optimal-policy learning under interference on large networks. Existing approaches to this problem require repeated observations of the same fixed network and struggle to scale in sample size beyond as few as fifteen connected units --- both limit applications. We show that common assumptions on the structure of interference enable a parsimonious linear parameterization of the reward function. We develop a scalable Thompson sampling algorithm that maximizes cumulative rewards on an $n$-node network while allowing for both nodes and edges to be sampled at each time period. We prove upper and lower bounds on Bayesian regret that imply near-optimality.  Simulation experiments show that our algorithm learns quickly and outperforms existing methods. The results close a key scalability gap between causal inference methods for interference and practical bandit algorithms, enabling policy optimization in large-scale networked systems.
\end{abstract}

\section{INTRODUCTION}

Sequential learning and decision making arises in contexts as varied as dynamic pricing in microeconomics \citep{ROTHSCHILD}, experimentation in online platforms \citep{wager_xu}, and adaptive treatment allocations in clinical research \citep{murphy_odtr}. Many such problems revolve around the decision of which individuals to treat. For example, an online marketplace may want to assign coupons to customers who otherwise would not purchase an item, or a ride-sharing company may want to assign bonuses to drivers who would otherwise leave the platform. Both examples involve maximizing the impact of treatment, but doing so requires first learning its effect. Multi-armed bandit (MAB) algorithms have emerged as an effective method to solve both the learning and maximization problems simultaneously.

Classic bandit formulations assume that the outcome of an individual is not impacted by the treatment status of others. This assumption is violated in many important contexts, such as epidemiology \citep{spillover_epid}, marketplaces \citep{munro2025treatmenteffectsmarketequilibrium}, and social networks \citep{Ogburn02012024}. We study \textbf{interference} or \textbf{spillover effects}, where treatments impact not only the treated units but also their peers. We focus on interference that spreads along a network, where the nodes of the network define individuals or arms and the edges specify connections. For example, online experimentation often occurs on social networks or internet marketplaces where interference impacts outcomes and the underlying network changes over time. 

Despite their relevance to such problems, little work on MABs exists in this setting; existing algorithms struggle to scale beyond networks of size $n\approx15$ \citep{agarwal2024, jamshidi2025}. Moreover, these methods require that the network remains fixed across time \textit{and} that observations of rewards across time are independent. The simultaneous assumptions of a fixed network with persistent individuals and temporal independence can lead to inconsistencies. For example, the outcomes of participants in a vaccine trial depend on their treatment statuses in earlier periods. In this work, we relax the assumption that the network remains fixed.
 
Existing algorithms fail to scale because they attempt to capture arbitrary interference instead of leveraging common structural assumptions from the causal inference literature. For example, in the case of $L$ treatment levels, the linear reformulation of \citet{agarwal2024} means that a node of degree $d$ is represented by $(L+1)^{d+1}$ unique parameters. A network with maximum degree $d_{\text{max}}$ thus has up to $n(L+1)^{d_{\text{max}}}$ unknown parameters that need to be estimated. Thus,  with $n=1000$, $d_{\text{max}} = 10$, and $L=2$, the reward function would have over a million unknowns. Estimating this many parameters, which grows linearly with $n$ and exponentially with node degree $d_{\text{max}}$, leads to computational complexity cubic in $n$ and  prohibitive data requirements.

The causal inference literature has studied interference on networks with particular emphasis on specifying assumptions that are both realistic and practical. However, this work is almost exclusively in the context of static experimentation with a focus on identification and inference as opposed to dynamic decision making. Motivated by applications where networks are large and dynamic, we address the gap between the MAB and causal inference literature by devising a scalable MAB algorithm. Our key insight is that by adopting common assumptions \citep{manski2013} about the additivity and symmetry of interference patterns, rewards form a system of $n$ linear equations, allowing for an efficient Thompson sampling algorithm capable of handling networks that are orders of magnitude larger than competing methods. 

Our contributions are as follows: we show that under common interference assumptions, the vector of node-level rewards becomes linear in a parameter vector $\boldsymbol{\theta}$. Within a broad class of reward function formulations, $\boldsymbol{\theta}$ is low-dimensional, enabling a scalable Thompson sampling algorithm that maximizes cumulative rewards. Our framework allows for both nodes and edges to be sampled at each time period, resolving potential contradictions in previous works that require both temporal independence and a fixed population of nodes. We prove both upper and lower bounds on regret, showing our algorithm is near-optimal. Finally, to underscore the flexible nature of our framework, we provide empirical results under a variety of reward functions and network specifications.

\section{BACKGROUND}
\paragraph{Interference and Peer Influence Effects} Researchers in causal inference and econometrics have dedicated considerable work to interference and peer influence effects. Many topics, such as identification \citep{manski_ident}, inference \citep{on_causal_interference, aronow_samii, toulis13}, and experimental design \citep{hudgens_holloran, EcklesKarrerUgander} have been covered in a variety of contexts, including networks (e.g., see \citet{annurev_survey} for a survey). However, little work combines an inference objective under interference with a maximization problem. Notable exceptions include \citet{wager_xu} and \citet{viviano}. Neither matches our setting, as \citet{wager_xu} considers a specific form of interference due to marketplace effects and \citet{viviano} is restricted to offline contexts.

\paragraph{Bandits with Interference} Interference is in part a dimensionality problem. Without restriction, in the case of a binary treatment and $n$ individuals, it increases the number of arms from $n$ to $2^n$. While the literature on combinatorial bandits (e.g., \citet{chen13a}) provides intuition for the difficulties of our problem related to the curse of dimensionality, assumptions typical of that literature (such as independent rewards) do not hold in our setting. Work on bandits under interference is relatively sparse. Spatial interference has been studied \citep{laber_opt_treat, jia2024multiarmedbanditsinterference} as well as network interference \citep{agarwal2024}. Work exists on optimal individual treatment regimes with interference  \citep{riu_op_ind_treat_reg} as well as linear contextual bandits under interference \citep{xu2024linearcontextualbanditsinterference}. Our algorithm draws its inspiration from the linear bandit literature \citep{imp_lin_algs} and Thompson sampling \citep{russo2014learning}. The work closest to ours is \cite{agarwal2024}. The authors design a MAB algorithm for a static network assuming only ``sparse network interference,'' and formulate the reward function as a linear model. That work does so indirectly through discrete Fourier analysis, while we show that the rewards can be written directly as a linear function of unknown parameters. We show that our method can outperform \citet{agarwal2024} even when our linear formulation is misspecified.

\begin{figure*}[t!]
\centering

\begin{tikzpicture}
    \coordinate (left_cluster) at (-4.5, 0);
    \coordinate (mid_cluster) at (0, 0);
    \coordinate (right_cluster) at (4.5, 0);

    \node[white_node] (L1) at (left_cluster) {};
    \node[red_node]   (L2) at ($(left_cluster) + (1.7, 1.2)$) {Jon};
   
    \node[white_node] (L3) at ($(left_cluster) + (2.1, -1.2)$) {};

    \node[red_node]   (M1) at (mid_cluster) {};
    \node[red_node]   (M2) at ($(mid_cluster) + (1.7, 1.2)$) {Jon};
    
    \node[white_node] (M3) at ($(mid_cluster) + (2.1, -1.2)$) {};
    
    \node[white_node] (R1) at (right_cluster) {};
    \node[red_node]   (R2) at ($(right_cluster) + (1.7, 1.2)$) {Jon};
    
    \node[white_node] (R3) at ($(right_cluster) + (2.1, -1.2)$) {};

    \draw[thick] (M1) -- (M2);
    \draw[thick] (M2) -- (M3);
    \draw[thick] (R1) -- (R2);
    \draw[thick] (R2) -- (R3);

    \coordinate (line_pos) at ($(L3.east)!0.5!(M1.west)$);
    \draw[dotted, thick] (line_pos |- M2) ++(0,0.5cm) -- (line_pos |- M3) --++(0,-0.5cm);

    \coordinate (eq_center) at ($(mid_cluster)!0.5!(right_cluster)$);
   
    \node[below=1.8cm of eq_center, text width=10cm, align=center] {
        \Large
        \[
        r_{t,\text{Jon}}\left(\raisebox{-.5\height}{\scalebox{0.2}{\nodenetwork[nolabel]{red_node}}}\right)
        \neq
        r_{t,\text{Jon}}\left(\raisebox{-.5\height}{\scalebox{0.2}{\nodenetwork[nolabel]{white_node}}}\right)
        \]
    };
\end{tikzpicture}

\caption{A visualization of network interference, where the reward function of ``Jon'' depends on the treatment status of his neighbors. On the left is the standard MAB case.}
\label{fig:network_function}
\end{figure*}

\section{MODEL SETUP}
At each time period $t$, the agent observes an adjacency matrix $\mathbf{A}_t$ for a set of $n$ nodes with maximum degree $d_{\text{max}}$. The agent must choose the treatment allocation vector $\mathbf{Z}_t \in \{0,1,\ldots,L\}^n$ potentially subject to budget constraint $B_t$ (e.g. an upper limit on the number of treated nodes) where $L$ denotes the number of treatment levels. Each node $i$ has a corresponding reward function $r_i(\mathbf{Z}_t;\mathbf{A}_t)$. The agent attempts to maximize the sum of the reward functions over time.

 Without placing assumptions on interference, each node-level reward function $r_i$ depends on the entire treatment vector $\mathbf{Z}_t$ as well as $\mathbf{A}_t$, which we illustrate in Figure \ref{fig:network_function}. For a given network $\mathbf{A}_t$, each reward function is thus a mapping $r_{i}(\cdot;\mathbf{A}_t):\{0,1,\ldots,L\}^n \to \mathbb{R}$. Learning each of these mappings simultaneously becomes computationally infeasible due to the input space growing exponentially in $n$. Therefore, we adopt a set of common assumptions on the structure of interference to allow for scalable optimization of the treatment policy.

\subsection{Interference Assumptions}
We adopt the neighborhood interference assumption \citep{sussman2017, belloni} for the node-level reward functions. Let $\mathcal{N}_{t,i}$ denote the set of neighbors of node $i$ at time $t$, i.e. for all $j \in \mathcal{N}_{t,i}$, $[\mbf{A}_{t}]_{i,j} = 1$.
\begin{assumption}
    For all nodes $i$, the reward function $r_{i}$ satisfies the neighborhood interference assumption (NIA) if for all treatment assignments $\mbf{Z}_t$, $\mbf{Z}_t'$ that agree on the set $\mathcal{N}_{t,i} \cup \{i\}$, $r_{i}(\mbf{Z}_t;\mbf{A}_t) = r_{i}(\mbf{Z}_t';\mbf{A}_t)$. \label{ass1}
\end{assumption}
NIA requires that a node's reward function depends only on its own treatment and the treatment status of its neighbors, reducing the dimension of the input space. For example, with $n=100$ and no assumptions on the interference pattern, the reward function of each node depends on the entire $100$-length vector of treatment assignments, implying $(L+1)^n$ possible inputs. Under NIA, a node with degree $d$ would have a reward function with $(L+1)^{(d+1)}$ inputs.

Under NIA, the vector of node-level reward functions can be represented as linear in a parameter vector $\boldsymbol{\theta}$ containing treatment and interference effect parameters \citep[eq.~4.1]{sussman2017}:
\begin{equation}
\mbf{r}_t = \mbf{H}(\mbf{Z}_t;\mbf{A}_t)\boldsymbol{\theta} + \boldsymbol{\epsilon}_t \label{linear_form}
\end{equation}

where $\mbf{H}(\mbf{Z}_t;\mbf{A}_t)$ is an $n\times p$ feature matrix that maps treatment assignments and network structure onto the effect parameters. We emphasize that linearity is not a modeling assumption but a reparameterization that directly follows from the neighborhood interference assumption (see Supplement Section 1 for details). However, without further assumptions, the dimension of $\boldsymbol{\theta}$ still inhibits scalable exploration and exploitation as each node has on the order of $(L+1)^{d+1}$ parameters. The following assumptions further restrict the form of interference.

\begin{assumption}
    The reward function $r_{t,i}$ satisfies additivity of main effects if \[r_{i}(Z_{t,i}, \mbf{Z}_{\mathcal{N}_{t,i}}) = r_{i}(Z_{t,i}, \boldsymbol{0}) + r_{i}(0,\mbf{Z}_{\mathcal{N}_{t,i}}).\] \label{ass2}
\end{assumption}

\begin{assumption}
    The reward function $r_{i}$ satisfies symmetrically received interference if, for all permutations $\tau$, $r_{i}(Z_{t,i}, \mbf{Z}_{\mathcal{N}_{t,i}}) = r_{i}(Z_{t,i}, \tau(\mbf{Z}_{\mathcal{N}_{t,i}}))$. \label{ass3}
\end{assumption}
We have dropped the dependence of $r_i$ on $\mathbf{A}_t$ for notational simplicity. Assumption \ref{ass2} states that the direct effect does not interact with the indirect effects. Assumption \ref{ass3} implies that a node's reward depends only on the number of treated neighbors, not which specific neighbors get treated. Together, Assumptions $1$ to $3$ are termed SANIA (Symmetric and Additive NIA). 

\subsection{Reward Function Specification}\label{subsec:reward_funcs}
The SANIA assumptions allow for a large class of reward functions; the specific node-level form can be tailored to the application, including information about the underlying network structure and the nodes themselves. 

A simple starting point is to assume that each node-level reward function has unique direct and indirect effect parameters that do not depend on context or network structure beyond the adjacency matrix. Allowing $d^1_{t,i}$ to be the number of treated neighbors of node $i$ at time $t$, this results in a reward function of the form
\begin{equation}r_i(\mbf{Z_t};\mbf{A_t}) = Z_{t,i}\cdot\mu_i + \sum_{k=1}^{d_i} \gamma_{i,k} \cdot \mbf{1}_{\{d^1_{t,i} = k\}} + \epsilon_{t,i}\label{sania_reward_fixed}
\end{equation}
where we leave out an intercept parameter for simplicity and assume treatment $Z_{t,i}$ is binary. Here, $d_i$ denotes the degree of node $i$ and $d_i^1$ the number of treated neighbors of node $i$. The parameters $\gamma_{i,k}$ determine the spillover effect of having $k$ treated neighbors. This is the form taken by \citet{agarwal2024} with further restrictions due to SANIA. 

We note two issues with this approach. First, the number of parameters scales linearly with the size of the network ($O(n\cdot d_{\text{avg}})$). While this improves upon the exponential scaling under NIA, it can be restrictive for large networks. Second, the use of fixed node-specific parameters implicitly assumes a static population observed over time. It thus becomes difficult to justify temporal independence in errors.

We provide alternative parameterizations that allow for information sharing across nodes and do not rely on fixed node identities. Our method does not rely on a specific parameterization but instead defines a flexible framework that researchers can use to construct models tailored to their application.

\subsubsection{Shared Parameters}
The simplest parameterization within this framework assumes that all nodes share a single set of parameters: $\mu_i = \mu_j$ and $\gamma_{i,k} = \gamma_{j,k}$ for all $i,j,k$. If all interference parameters $\gamma_k$ are $0$, this reduces to the standard individualistic treatment response assumption. With non-zero interference, it becomes the constant treatment response in the depth of interference model \citep{manski2013}.

This approach is highly scalable but has the potential to be misspecified. Despite this sensitivity, it is a common approach in the causal inference literature \citep{toulis13, EcklesKarrerUgander} and can be aided by including covariates in the model.

\subsubsection{Grouped Parameters}\label{grouped_params}
A natural extension of the previous parameterization assumes that groups of units (either observed or latent) share parameters. This model is based on the observations in sociology and network science that similarly behaving nodes tend to share connections \citep{birds_feathers}. All nodes in group $g$ share a parameter vector $\boldsymbol{\theta}_g = \{\mu_g,\gamma_{g,1},\ldots,\gamma_{g,m}\}$. These combine to form the full parameter vector: $\{\boldsymbol{\theta}_1^\top,\ldots,\boldsymbol{\theta}_G^\top\}^\top$.

\subsubsection{Latent Features and Other Generalizations}
A further generalization embeds the networks in a Euclidean space. The distance between nodes in this space determines the probability of forming links. Using the latent space model of \citet{Hoff}, positions in the unobserved space can be estimated for each node and included in their reward function. For example, assuming each node has a corresponding two-dimensional location vector $\boldsymbol{\zeta}$, the reward functions could take the form

\[r_{i}(\mbf{Z}_t;\mbf{A}_t,\boldsymbol{\zeta}_{t,i}) = \boldsymbol{\zeta}_{t,i}^\top\boldsymbol{\mu} + \sum_{k=1}^{d_i} (\boldsymbol{\zeta}_{t,i}^\top \boldsymbol{\alpha}_k) \cdot \mbf{1}_{\{d^1_{t,i} = k\}} + \epsilon_{t,i}.\]

\subsection{Design Matrix}
The structure of the design matrix $\mbf{H}(\mbf{Z}_t;\mbf{A}_t)$ is determined by the chosen reward parameterization. Its rows  consist of indicator variables derived from the treatment vector $\mbf{Z}_t$, features summarizing the interactions between $\mbf{Z}_t$ and $\mbf{A}_t$ (e.g. the counts of treated neighbors), and additional features such as group labels.

For example, under the shared parameters model, the portion of $\mbf{H}$ corresponding to the direct effect $\mu$ would be a single column containing the indicators $Z_{t,i}$.

Under the grouped parameters model, there would be a set of columns for the direct effects; the column for group $g$ would have an entry of $1$ for node $i$ if it is in group $g$ and is treated, and $0$ otherwise.

While each row is simple to compute, the mapping itself can be complex and thus maximizing the rewards with respect to $\mbf{Z}_t$ conditional on a set of parameters can be difficult. We discuss this issue further in later sections. 

\subsection{Regret}
The agent seeks to maximize the cumulative node-level rewards with respect to the treatment vector. To measure the agent's performance, we use the Bayesian regret---the expected gap in cumulative rewards under the optimal policy and the agent's actions with respect to the prior distribution over $\boldsymbol{\theta}$. The optimal treatment at time $t$, denoted $\mbf{Z}_t^*$, maximizes the sum of expected node-level rewards:

\[\mbf{Z}_t^* = \arg\max_{\mbf{Z}_t} \sum_{i=1}^n\mathbb{E}\left[ r_{i}(\mbf{Z}_t;\mbf{A}_t)\right]\]

The regret is then the cumulative expected difference in rewards between the optimal policy and the agent's policy:
\[\text{Reg}_T =\sum_{t=1}^T\sum_{i=1}^n\mathbb{E}\left[r_i(\mbf{Z}_t^*;\mbf{A}_t) - r_i(\mbf{Z}_t;\mbf{A}_t) \right] .\]
In the following section, we define an algorithm that achieves low regret with high probability and provide upper and lower bounds on regret.

\section{THOMPSON SAMPLING UNDER NETWORK INTERFERENCE}
We devise a scalable Thompson sampling algorithm that can accommodate any reward function satisfying the neighborhood interference assumption, such as those discussed in the previous section. The parameterization of the reward functions will determine the dimension of $\boldsymbol{\theta}$, which will in turn govern the algorithm's bounds and computational complexity. We discuss in detail the impact of $\boldsymbol{\theta}$ on the theoretical and empirical results of our algorithm, thus providing a flexible and theoretically grounded framework for maximizing treatment policies under network interference. 

\subsection{Thompson Sampling Algorithm}

We begin with the standard assumption on the noise $\boldsymbol{\epsilon}_t$. This assumption facilitates the theoretical development in our paper and naturally maps onto the linear specification in \citet{sussman2017}, which includes an individualized baseline effect (see additional details in the Supplement):

\begin{assumption}
    The terms $\epsilon_{t,i}$ are independent 1-sub-Gaussian random variables for all $t$, $i$. \label{ass4} 
\end{assumption}

Thompson sampling requires a prior distribution over the parameter vector $\boldsymbol{\theta}$ of dimension $D$. Motivated by normal-normal conjugacy in Bayesian linear regression, we specify that $\boldsymbol{\theta}\sim \text{N}\left(\boldsymbol{0},\frac{1}{\lambda} \mbf{I}_{D}\right)$. Having received an observation $\mbf{r}_t$, the posterior is updated using the standard conjugate posterior formulas \citep{gelman2013bayesian}. The hyperparameter $\lambda$ is effectively equivalent to the penalty term in ridge regression. Because its effect diminishes in the number of observations, it can be ignored in theoretical analysis, but it can have considerable impact on regret in early time periods. We discuss this further in Section 3 of the Supplement.

Our algorithm thus maintains a posterior distribution $\pi(\boldsymbol{\theta} | \mbf{r}_1, \ldots, \mbf{r}_t)$ over the unknown parameters $\boldsymbol{\theta}$. At each time period, the agent draws $\boldsymbol{\theta}^{(t)}$ from the posterior and chooses action $\mbf{Z}_t$ that maximizes the total reward conditional on $\boldsymbol{\theta}^{(t)}$.

\begin{algorithm}
\caption{Thompson Sampling under Interference}
\begin{algorithmic}[1]
\State \textbf{Input:} Prior mean $\boldsymbol{\mu}_0$, prior covariance $\boldsymbol{\Sigma}_0$, regularization $\lambda$, noise variance $\sigma^2$
\For{$t = 1$ to $T$}
    \State Observe network $\mbf{A}_t$ and any features
    \State Sample $\boldsymbol{\theta}^{(t)} \sim \mathcal{N}(\boldsymbol{\mu}_{t-1}, \boldsymbol{\Sigma}_{t-1})$
    \State Choose treatment vector $\mbf{Z}_t$:
    \[\mbf{Z}_t = \arg\max_{\mbf{Z}} \mbf{1}_n^\top[\mbf{H}(\mbf{Z};\mbf{A}_t)\boldsymbol{\theta^{(t)}}]\]
    \State Observe rewards $\mbf{r}_t$
    \State Compute $\mbf{H}_t = \mbf{H}(\mbf{Z}_t;\mbf{A}_t)$
    \State Update posterior:
    \Statex \qquad $\boldsymbol{\Sigma}_t = \left(\mbf{H}_t^\top \mbf{H}_t/\sigma^2 + \boldsymbol{\Sigma}_{t-1}^{-1} \right)^{-1}$
    \Statex \qquad $\boldsymbol{\mu}_t = \boldsymbol{\Sigma_{t}}\left(\mbf{H}_t^\top \mbf{r}_t/\sigma^2 + \boldsymbol{\Sigma}_{t-1}^{-1}\boldsymbol{\mu}_{t-1} \right)$
\EndFor
\end{algorithmic}
\label{alg:TS-Interference}
\end{algorithm}

Although our algorithm resembles Thompson sampling for classical linear bandits in its form, its application to our setting is not straightforward. In classical linear bandits, a single reward observation is modeled as $r_t = \langle X_t, \boldsymbol{\theta} \rangle + \epsilon_t$. In contrast, our setting involves $n$ interdependent reward observations per round, each influenced by network interference. The interdependence means that it is impossible to simply perform $n$ independent linear bandit updates. Our reformulation enables the extension of scalable linear bandit algorithms to a regime where prior work has failed to scale beyond the smallest of networks ($n \approx 10).$  The algorithm requires new regret analysis, which we provide in the following section.

In practice, the agent may be constrained by a budget. For example, they may only be able to treat $B<n$ units per round. This is easily accommodated by our algorithm and does not impact the theoretical results, as it simply redefines the optimal policy as a maximum over the constrained action space. 

\paragraph{Computational Complexity}
Allowing $D$ to denote the dimension of $\boldsymbol{\theta}$, line 8 of Algorithm \ref{alg:TS-Interference} has time complexity $O(nD^2 + D^3)$.  For comparison, a regression on the most general NIA formulation has complexity $O(n^32^{2d_{\text{max}}} + n^32^{3d_{\text{max}}})$. This scales with $n^3$ (and exponentially in $n$ if $d_{\text{max}}$ increases with $n$), meaning that existing explore-then-commit algorithms (e.g. \cite{agarwal2024}) fail to scale with $n$.

The maximization problem at Line 5 of Algorithm \ref{alg:TS-Interference} is nontrivial and  is the primary computational bottleneck during implementation. We used the Gurobi optimization software \citep{gurobi} in our experiments. This problem has not yet been discussed in the network bandits literature, as existing algorithms fail to scale beyond small networks where maximization is easy. 

\subsection{Regret Analysis}

We first provide an upper bound on the Bayesian regret of Algorithm \ref{alg:TS-Interference}.

\begin{theorem}
    Assume there exist positive constants $c_1$, $c_2$ such that $\sup_{\boldsymbol{\theta} \in \Theta} \| \boldsymbol{\theta}\|_2\leq c_1$ and $\sup_{\mbf{H} \in \mathcal{H}^n} \|\mbf{H}\|_2 \leq c_2$, and suppose Assumptions \ref{ass1}--\ref{ass4}
    hold. Algorithm \ref{alg:TS-Interference} then satisfies 
    \[\text{BayesRegret}(n,T) = O\left(D\sqrt{nT}\log(nT) \right)\]
    \label{main_thm}
\end{theorem}
We prove Theorem \ref{main_thm} by adapting the results on frequentist regret of \citet{imp_lin_algs} to our setting and then applying Proposition $5$ of \citet{russo2014learning}. Our assumptions on $\boldsymbol{\theta}$ and $\mbf{H}$ match those of \citet{russo2014learning}. Full proofs are provided in Section 2 of the Supplement. 

We show that our approach is near-optimal by deriving a lower bound on regret.

\begin{theorem}
    Under the same assumptions as Theorem \ref{main_thm}, for any policy $\pi$, there exists a $\theta \in \Theta$  such that $R_T(\mathcal{A}, \theta)\geq \Omega\left(D\sqrt{nT}\right).$
\end{theorem}

The proof adapts arguments from Section 24.1 of \cite{Lattimore_Szepesvári_2020}. We see that the upper bound provided in Theorem \ref{main_thm} matches this lower bound up to a logarithmic factor, thus proving that our algorithm is near-optimal.

Interestingly, our regret bounds resemble those for standard linear bandits but include a factor of $n$. The intuition for this difference lies in the feedback structure: for a network of size $n$, by time $T$ we have accumulated information for $nT$ nodes while the standard linear bandit algorithms will have only received $T$ observations. However, the increased complexity of the action space partially offsets the increase in information at each round. The balance of these opposing forces mirrors the exploration-exploitation tradeoff in standard linear bandits, resulting in our bounds.

As discussed in Section \ref{subsec:reward_funcs}, a potential pitfall is that $D$ may scale with $n$. For example, if each node has a unique direct effect, then the dimension of $\boldsymbol{\theta}$ scales at least linearly in $n$.  We focus our theoretical analysis on the case of a fixed $D$ and discuss the impact of $T$ and $n$ on regret as this allows a clear comparison to other algorithms in the bandit literature. We provide results from a variety of choices of the reward function through simulation studies in the following section. 

\paragraph{Comparison to previous approaches}

\begin{itemize}
    \item \textbf{Classic MAB Approach.} Standard MAB algorithms fail in this context for two reasons. First, they would attempt to learn each of the $(L+1)^n$ arms independently, which quickly becomes infeasible (e.g. TS has a regret bound of $O\left(\sqrt{2^nT\log2^n}  \right)$ \citep{near_opt_ts}). Second, they can only be applied to a static network, as the node reward functions are fixed over time.
    \item \textbf{Linear Bandits.} Collapsing Equation \eqref{linear_form} by summing over the rows results in a standard linear bandit formulation. However, this turns $n$ data points into $1$, losing a significant amount of information about $\boldsymbol{\theta}$. By learning more from the same data, our algorithm performs strictly better. 
    \item \textbf{Network Bandits.} Existing approaches for bandits under network interference assume only NIA. The algorithm in \cite{agarwal2024} attempts to estimate the Fourier coefficients of the reward function for each node. The length of the resulting parameter vector is exponential in the size of the network, severely limiting scaling ability. The inability of existing algorithms to scale is a primary motivation for our work.    
\end{itemize}

\section{SIMULATIONS}
We validate the flexibility and scalability of our method through simulation studies. We structure our experiments around three reward function parameterizations: a  shared parameter model, a block-structured model, and a more complex latent features model. 

\subsection{Linear Spillovers}\label{sec:lin_spill}
We begin with a shared-parameter model. Each node has reward function 
\[r_i(\mbf{Z}_t;\mbf{A}_t) = \mu\cdot Z_{t,i} + \sum_{k=1}^{d_i} \gamma_{k} \cdot \mbf{1}_{\{d^1_{t,i} = k\}} + \epsilon_{t,i}.\]

We set $\boldsymbol{\mu}_0=0$,$\Sigma_0=\mbf{I}$, and use budget $B=\frac{n}{5}$ (i.e. $\|\mbf{Z}_t\|_1 \leq B$). For each iteration, we draw $\mu \sim \text{N}(1,0.2)$ and $\gamma_k\sim \text{N}(k, 0.5)$. At each time period, we simulate $\mbf{A}_t$ from a stochastic block model (SBM) of sizes $n \in \{100, 500, 1000\}$ with $K = \frac{n}{10}$ groups having uniform group membership probabilities, the probability $p_{\text{within}} = 0.25$ of sharing a link with nodes in the same group and probability $p_{\text{between}} = \frac{1}{n}$ of a link between nodes in different groups. Errors are distributed $\epsilon_{i,t}\sim \text{N}(0,1)$.

For each $n$, we run $25$ iterations with $T=100$. In Figure \ref{fig:linear_cumulative} we evaluate our algorithm using the mean regret of the iterations and include $95\%$ confidence bands. 

\begin{figure}[htbp]
    \centering
    \includegraphics[width=0.95\linewidth]{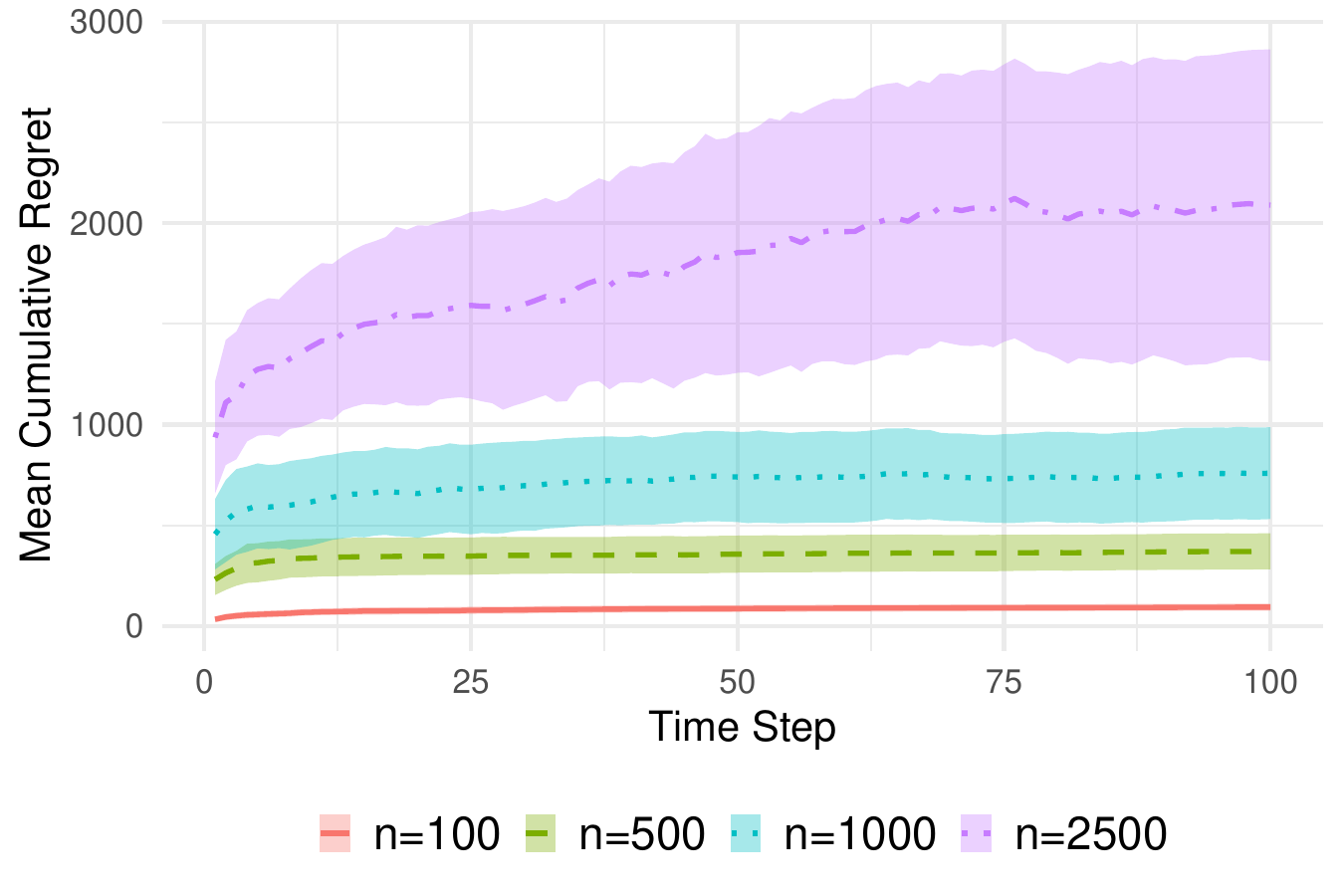}
    \caption{Cumulative regret plot under the shared parameters model.}
    \label{fig:linear_cumulative}
\end{figure}

\subsection{Grouped Effects}\label{grouped_effects_sim}
We now consider the grouped model of Section \ref{grouped_params}. For each group, the reward function follows Equation~\eqref{sania_reward_fixed}, with $\mu_g$ representing the direct effect for group $g$ and $\gamma_g$ representing the vector of potential spillover effects for units in group $g$:
\[r_i(\mbf{Z_t};\mbf{A_t},G_{t,i}=g) = Z_{t,i}\cdot\mu_g + \sum_{k=1}^{d_i} \gamma_{g,k} \cdot \mbf{1}_{\{d^1_{t,i} = k\}}+ \epsilon_{t,i}.\]

We use the same data generating process for $\mbf{A_t}$ as the previous section. The parameter groupings at each time $t$ are defined by the $K$ groups of the SBM which we for now assume the agent observes. 

For each iteration, we draw parameters $\mu_g \sim \text{N}(1, 0.2)$, $\gamma_{g,k} \sim \text{N}(k,1)$. The priors, budget, and number of iterations are the same as the previous section. We plot the cumulative regret for multiple network sizes in Figure \ref{fig:grouped_cumulative}.

\begin{figure}[htbp]
    \centering
    \includegraphics[width=0.95\linewidth]{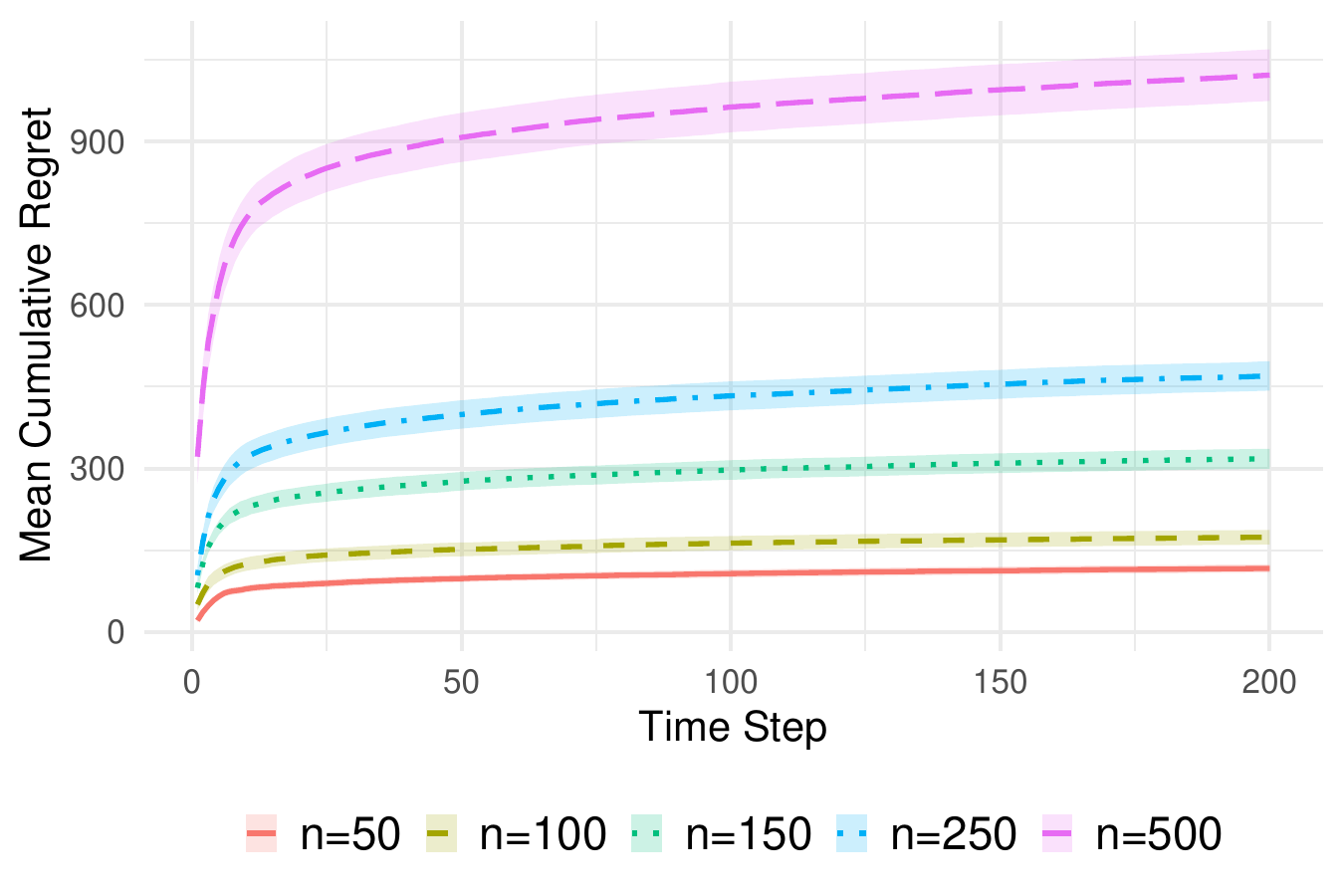}
    \caption{Cumulative regret for a variety of network sizes under grouped SANIA reward functions.}
    \label{fig:grouped_cumulative}
\end{figure}

\subsubsection{Estimating Group Labels}
We now consider the setting where the agent does not observe the group labels of all nodes. We assume that the agent observes the true label of one ``anchor'' node per group and must estimate all other group labels. 

To do so, the agent aggregates all previously observed networks into a cumulative adjacency matrix $\mbf{A}^{(t)}_{\text{cumulative}} = \sum_{i=1}^t \mbf{A}_i$. An SBM with a Poisson likelihood is then fit to this matrix to cluster the nodes into $K$ groups, producing an estimated group label for each node. The anchor nodes are key to avoid label switching, as the labels themselves are arbitrary once the groups have been defined. 

While our regret bounds assume that the reward functions are correctly specified (and thus in this case node labels are observed), we show in Figure \ref{fig:sbm_est_groups} that our algorithm still performs well when some labels are latent and must be estimated. 

\begin{figure}[htbp]
    \centering
    \includegraphics[width=0.95\linewidth]{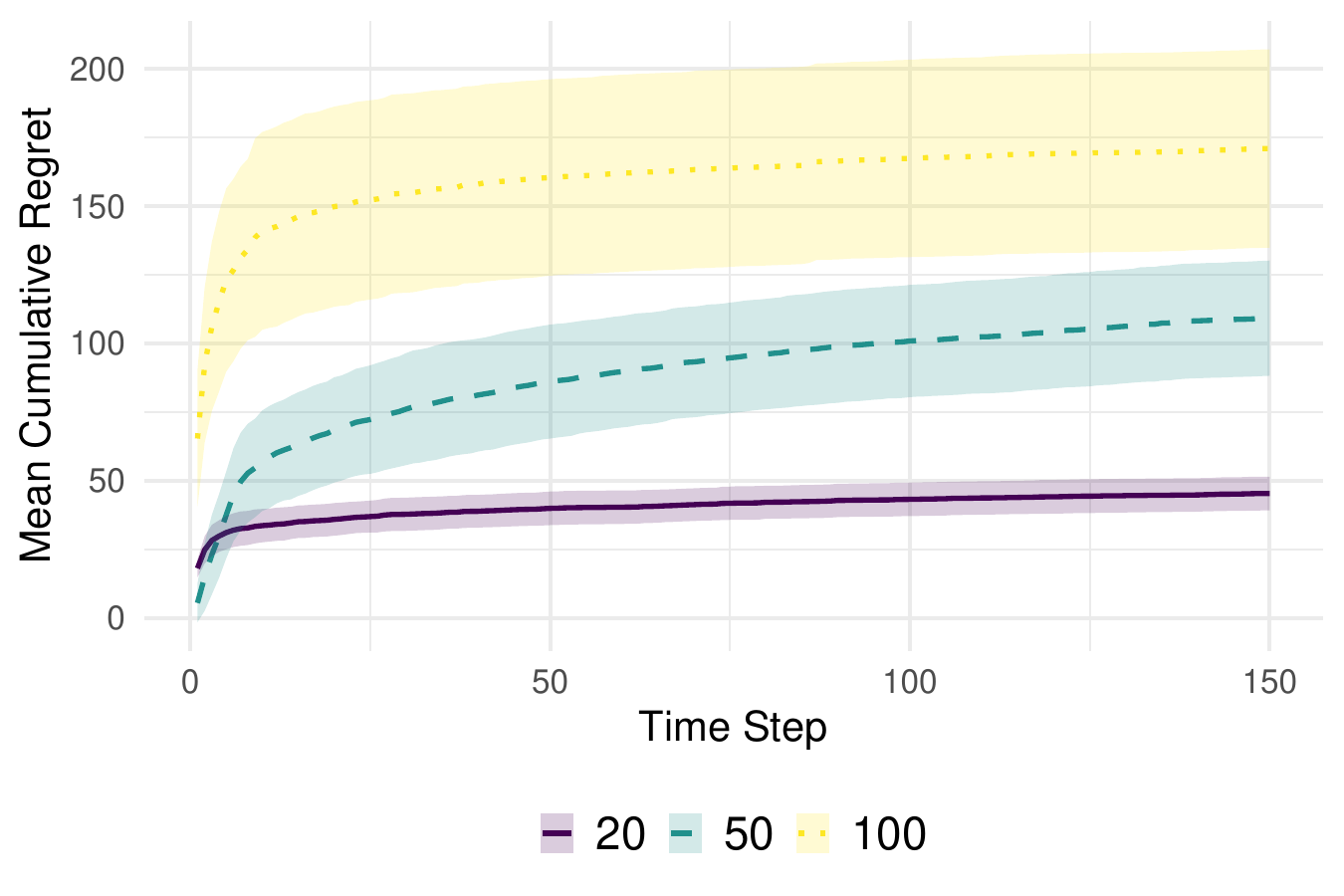}
    \caption{Cumulative regret under grouped SANIA reward function where group labels are estimated using one anchor node per group.}
    \label{fig:sbm_est_groups}
\end{figure}

\subsection{Heterogeneous Effects with Latent Features}

We now consider a model where treatment effects are heterogeneous and depend on latent node characteristics.  We use the latent feature parameterization from Section \ref{subsec:reward_funcs}, where the reward for node $i$ with latent features $\boldsymbol{\zeta}_i \in \mathbb{R}^2$ is given by:
$$r_{i}(\mbf{Z}_t;\mbf{A}_t,\boldsymbol{\zeta}_{i}) = Z_{t,i}\cdot \boldsymbol{\zeta}_{i}^\top\boldsymbol{\mu} + \sum_{k=1}^{d_{\text{max}}} (\boldsymbol{\zeta}_{i}^\top \boldsymbol{\alpha}_k) \cdot \mbf{1}_{\{d^1_{t,i} = k\}} + \epsilon_{t,i}.$$

The data-generating process uses the same underlying features to drive both network structure and rewards:
\begin{enumerate}
    \item \textbf{Latent Feature Generation:} For each iteration, we generate latent features for $n\in \{100, 500, 1000\}$ nodes by drawing $\boldsymbol{\zeta}_i \sim \mathcal{N}(\mathbf{0}, \mathbf{I}_2)$ for $i=1,\dots,n$. These features are fixed for all $t$.
    \item \textbf{Parameter and Reward Generation:} The true parameters are drawn once per run as $\boldsymbol{\mu} \sim \mathcal{N}(\mathbf{0}, \mathbf{I}_2)$ and $\boldsymbol{\alpha}_k \sim \mathcal{N}(\mathbf{0}, \mathbf{I}_2)$ for $k=1, \dots, d_{\text{max}}$. The noise is standard normal, $\epsilon_{t,i} \sim \mathcal{N}(0,1)$.
    \item \textbf{Network Generation:} At each time step $t$, the network $\mbf{A}_t$ is generated from a latent space model. The probability of an edge between nodes $i$ and $j$ depends on the Euclidean distance between their features: $P([\mbf{A}_t]_{ij} = 1) = \text{logit}^{-1}(\alpha - \|\boldsymbol{\zeta}_i - \boldsymbol{\zeta}_j\|_2)$. This process creates networks with community structures where nearby nodes in the latent space are more likely to be connected.
    
\end{enumerate}

For this experiment, we assume the latent features $\boldsymbol{\zeta}_i$ are observable to the agent at decision time. The agent's task is to learn the unknown parameter vectors $\boldsymbol{\mu}$ and $\{\boldsymbol{\alpha}_k\}$. We run the simulation for $T=500$ steps with a budget to treat $20\%$ of nodes. For each $n$, we set $\alpha$ so that the expected degree of each node is approximately $5$. We show the cumulative regret for multiple network sizes in Figure \ref{fig:latent_cumulative}.

\begin{figure}[htbp]
    \centering
    \includegraphics[width=0.95\linewidth]{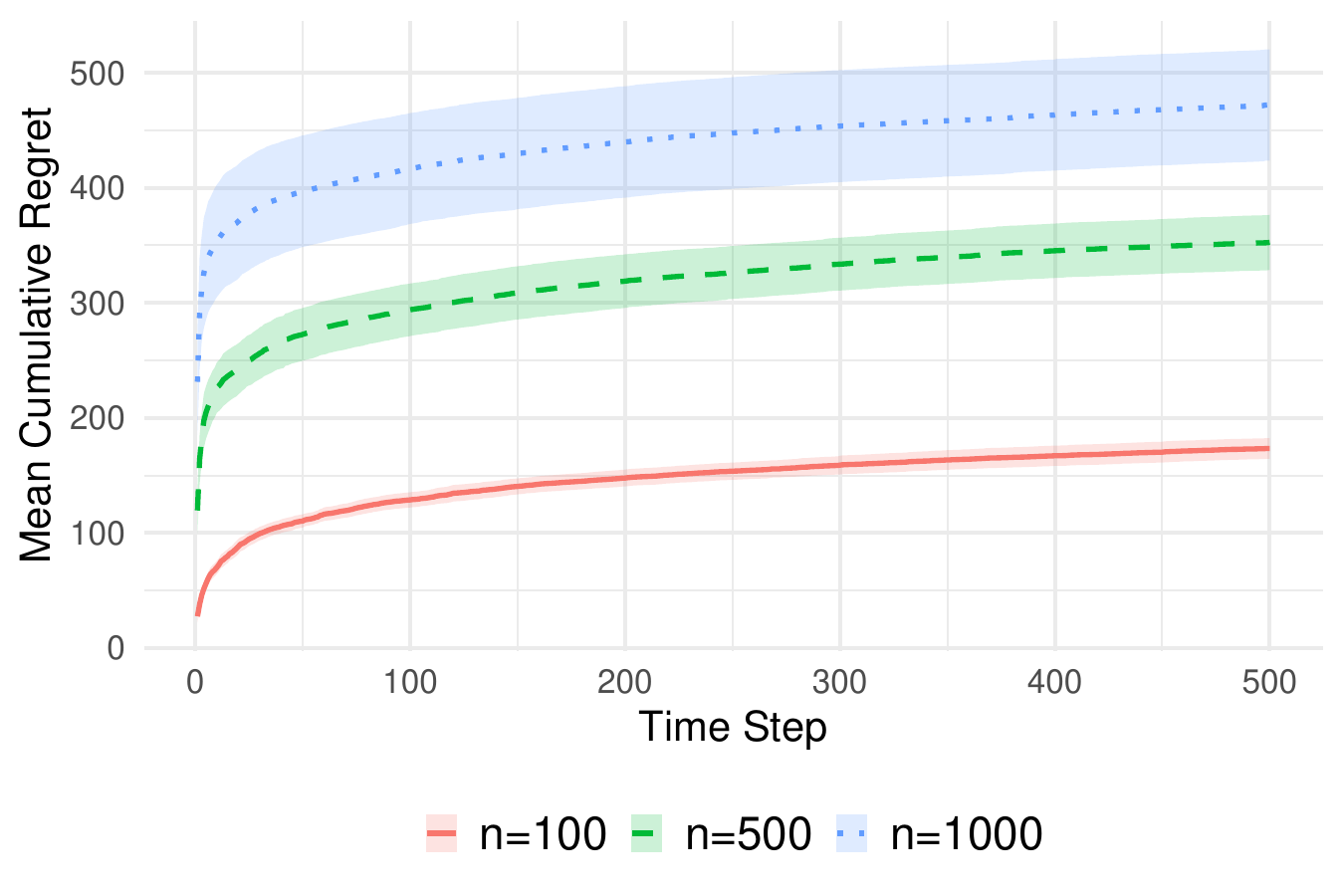}
    \caption{Cumulative regret plots for latent feature reward model.}
    \label{fig:latent_cumulative}
\end{figure}

\subsection{Comparisons to Existing Methods}
To compare our method to \citet{agarwal2024} we postulate a most favorable data generating process (DGP) to their approach: the reward functions follow Equation \ref{sania_reward_fixed} and remain fixed over time, and the network is small ($n=8$) and static. We use the same parameter generation scheme as the grouped effects simulations. 

\begin{figure}[h]
    \centering
    \includegraphics[width=0.95\linewidth]{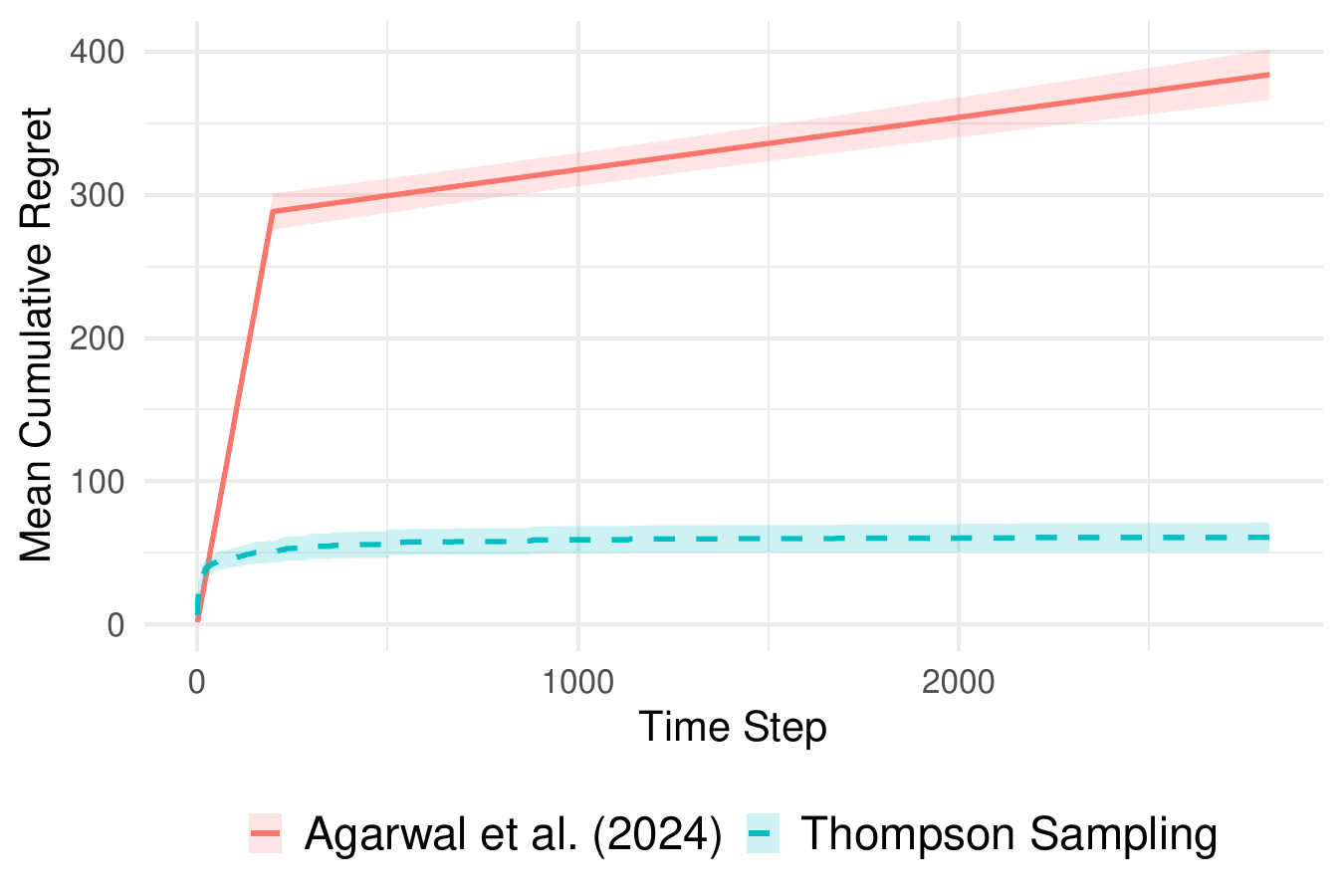}
    \caption{Comparison to \cite{agarwal2024} on a static network of size $n=8$.}
    \label{fig:agarwal_compare}
\end{figure}

In Figure \ref{fig:agarwal_compare}, we see that our method significantly outperforms the more general \citet{agarwal2024}. We also compare our methods under misspecification of the SANIA assumptions, specifically by omitting Assumption $2$ from the DGP . In Figure \ref{fig:agarwal_misspec}, we see that our method continues to outperform, despite misspecification of the SANIA assumptions.

\begin{figure}[h]
    \centering
    \includegraphics[width=0.95\linewidth]{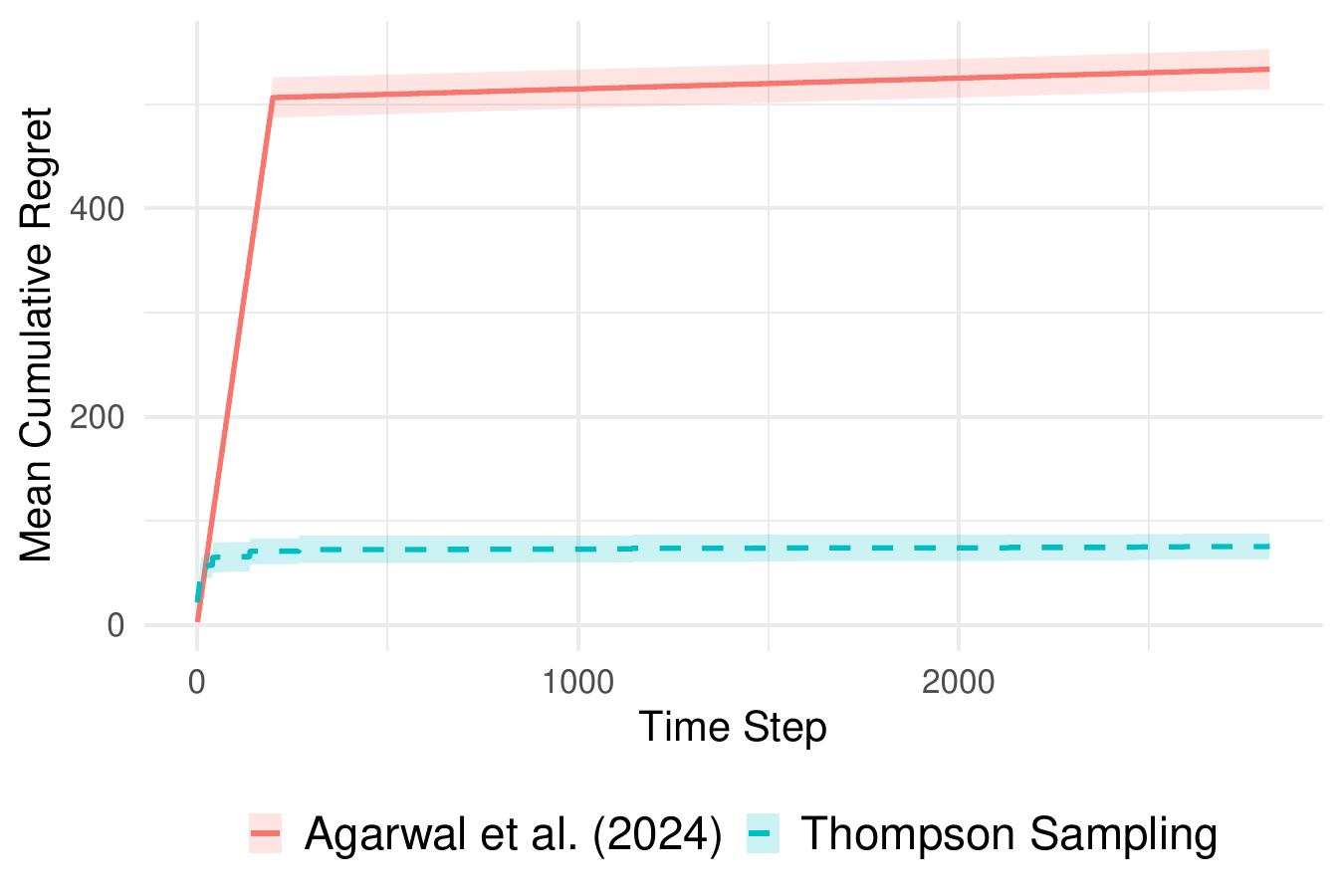}
    \caption{Comparison to \cite{agarwal2024} under misspecification of the additivity assumption. Our method remains robust and continues to outperform.}
    \label{fig:agarwal_misspec}
\end{figure}

\subsection{SANIA Misspecification}
We investigate the sensitivity of our algorithm to violations of the SANIA assumptions. The following reward function violates Assumption~\ref{ass3} as the contribution of a treated neighbor depends on that neighbors' own degree:
\[r_i = \mu_i Z_i + \sum_{j \in \mathcal{N_i}}A_{ij}Z_j\deg(j) + \epsilon_i.\]
We generate rewards from this model and fix an Erd\"os-Renyi graph with $n=8$ and $p=0.3$. We use the reward model from Section~\ref{sec:lin_spill} which assumes symmetric spillovers and is thus misspecified.

In Figure~\ref{fig:sym_miss_comp} we compare directly to both \citet{agarwal2024} and a standard MAB, as these more general algorithms do not suffer from misspecification. We fix a network over time as required by \citet{agarwal2024}. Despite being misspecified, our algorithm significantly outperforms both alternatives. This is because the SANIA-based parameterization, while unable to capture the degree-dependent spillover effects exactly, learns an effective approximation with far fewer parameters. In contrast, \citet{agarwal2024} must estimate an exponentially larger parameter vector, and the standard MAB must explore $2^n$ arms independently, resulting in substantially higher variance that dominates any bias advantage from correct specification.

\begin{figure}[htpb]
    \centering
    \includegraphics[width=0.95\linewidth]{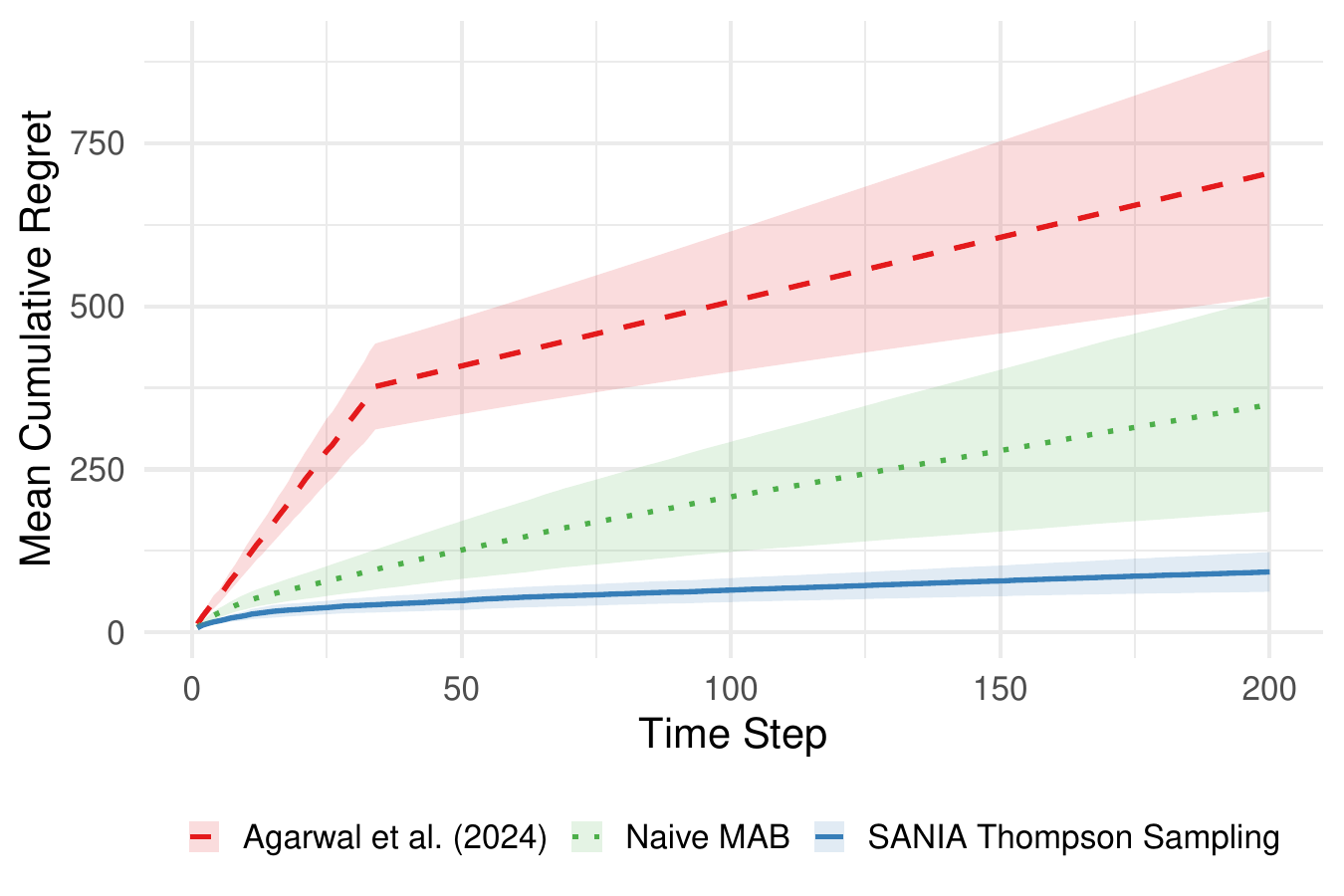}
    \caption{Comparison of Algorithm~\ref{alg:TS-Interference} to \citet{agarwal2024} under symmetry misspecification. }
    \label{fig:sym_miss_comp}
\end{figure}

\subsection{Network Misspecification}
Our algorithm and analysis assumes that the agent observes $\mbf{A}_t$ at each time period. However, issues with data collection can cause existing edges to be unobserved or nonexistent edges falsely set to $1$. We test our algorithms sensitivity to misspecified adjacency matrices by rerunning the experiment from Section~\ref{sec:lin_spill} while randomly flipping a portion $\epsilon\in  \{0, 0.05, 0.1, 0.2 \}$ of the edges in $\mbf{A}_t$. In Figure~\ref{fig:noisy_net} we see that regret scales well even with $20\%$ of the edges flipped.

\begin{figure}[h]
    \centering
    \includegraphics[width=0.95\linewidth]{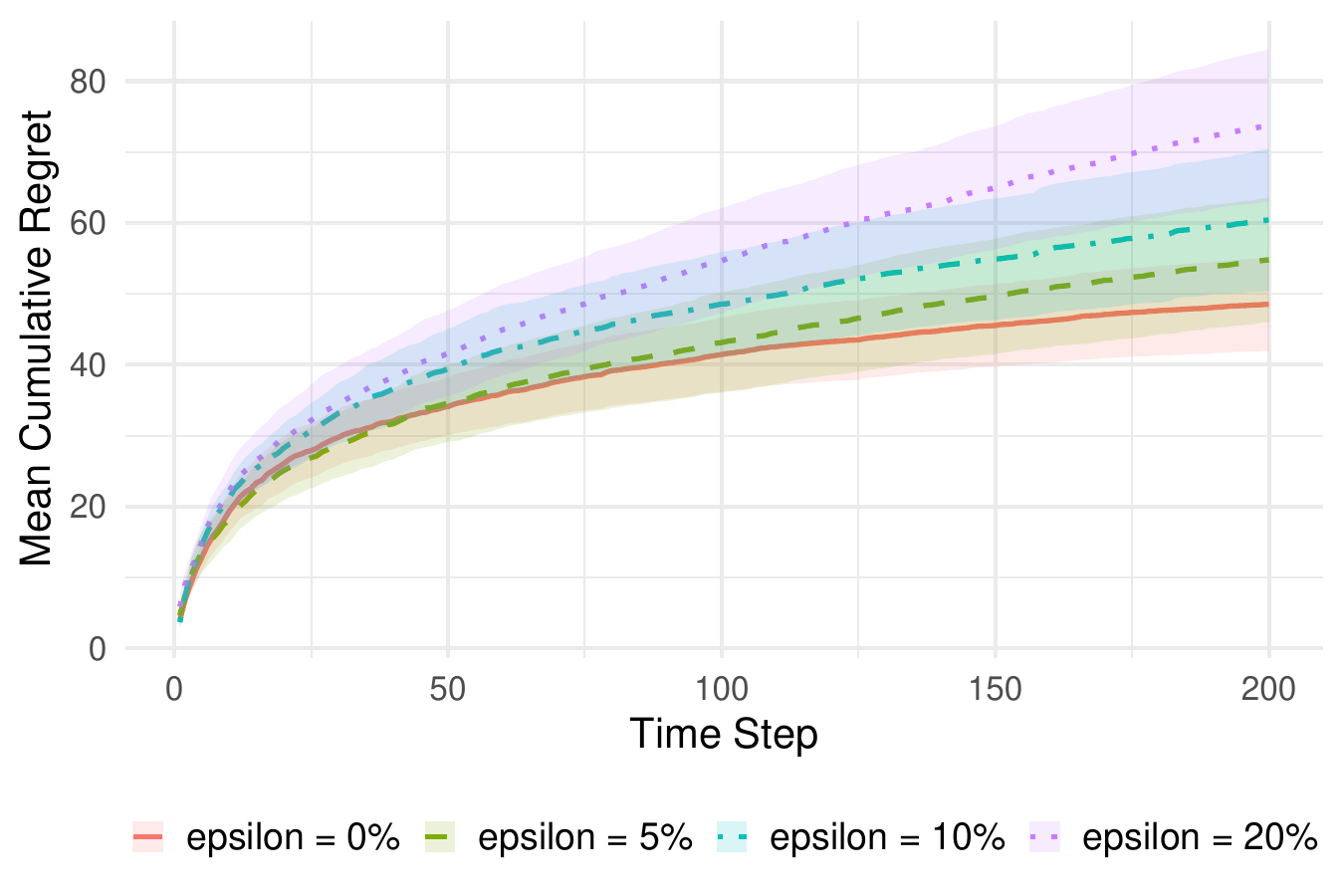}
    \caption{Cumulative regret under the linear spillover model for $n=8$ under various noise levels that result in misspecifed $\mbf{A}_t$.}
    \label{fig:noisy_net}
\end{figure}

\section{CONCLUSION} 
We have introduced a scalable Thompson sampling algorithm for regret minimization under network interference. We show that the SANIA assumptions can be leveraged for scalable policy maximization, bridging a gap between the causal inference and bandits literature. We prove Bayesian regret bounds and provide strong evidence of performance and scalability through simulation experiments. Our work suggests future research into linear optimization algorithms specific to network interference. Extensions of our method to partially observed networks could greatly impact their practicality. Empirical validations beyond simulation experiments can test the impact on real-world outcomes. 

\section*{Acknowledgments}
We gratefully acknowledge funding from the National Science foundation under grants DMS 2346292, DMS 2230074, DMS CAREER 2046880, and the National Institute of Health under grant R01CA280970.

\newpage

\bibliography{papers}





\section*{Checklist}
\begin{enumerate}

  \item For all models and algorithms presented, check if you include:
  \begin{enumerate}
    \item A clear description of the mathematical setting, assumptions, algorithm, and/or model. Yes (Sections 3 and 4)
    \item An analysis of the properties and complexity (time, space, sample size) of any algorithm. Yes (Section 4)
    \item (Optional) Anonymized source code, with specification of all dependencies, including external libraries. Yes, in Supplement.
  \end{enumerate}

  \item For any theoretical claim, check if you include:
  \begin{enumerate}
    \item Statements of the full set of assumptions of all theoretical results. Yes (Section 4, and proofs in Supplement)
    \item Complete proofs of all theoretical results. Yes (Supplement)
    \item Clear explanations of any assumptions. Yes (Sections 3 and 4)     
  \end{enumerate}

  \item For all figures and tables that present empirical results, check if you include:
  \begin{enumerate}
    \item The code, data, and instructions needed to reproduce the main experimental results (either in the supplemental material or as a URL). Yes.
    \item All the training details (e.g., data splits, hyperparameters, how they were chosen). Yes.
    \item A clear definition of the specific measure or statistics and error bars (e.g., with respect to the random seed after running experiments multiple times). Yes.
    \item A description of the computing infrastructure used. (e.g., type of GPUs, internal cluster, or cloud provider). Yes (Supplement)
  \end{enumerate}

  \item If you are using existing assets (e.g., code, data, models) or curating/releasing new assets, check if you include:
  \begin{enumerate}
    \item Citations of the creator If your work uses existing assets. Yes.
    \item The license information of the assets, if applicable. Yes.
    \item New assets either in the supplemental material or as a URL, if applicable. Yes.
    \item Information about consent from data providers/curators. Not Applicable.
    \item Discussion of sensible content if applicable, e.g., personally identifiable information or offensive content. Not Applicable.
  \end{enumerate}

  \item If you used crowdsourcing or conducted research with human subjects, check if you include:
  \begin{enumerate}
    \item The full text of instructions given to participants and screenshots. Not Applicable.
    \item Descriptions of potential participant risks, with links to Institutional Review Board (IRB) approvals if applicable. Not Applicable.
    \item The estimated hourly wage paid to participants and the total amount spent on participant compensation. Not Applicable. 
  \end{enumerate}

\end{enumerate}

\clearpage
\appendix
\thispagestyle{empty}

\onecolumn
\aistatstitle{Supplementary Materials}

\section{Linearity Under NIA}\label{nia_linear}
The neighborhood interference assumption (NIA) states that a node's outcome depends only on its own treatment status $Z_{t,i}$ and the treatment status of its neighbors $Z_{\mathcal{N}_{t,i}}$. As shown in Section 4.1.1 of \cite{sussman2017}, this directly implies that the reward functions can be written as 
\[r_i(\mbf{Z}_t;\mbf{A}_t) = \alpha_i + \beta_i Z_{t,i} + \Gamma_i(\mbf{Z}_{\mathcal{N}_{t,i}}) + Z_{t,i}\Delta_i(\mbf{Z}_{\mathcal{N}_{t,i}})\]
where 
\begin{enumerate}
    \item $\alpha_i$ is the baseline effect when neither the unit or any of its neighbors is treated.
    \item $\beta_i$ is the direct treatment effect.
    \item $\Gamma_i$ is the interference effect determined by the treatment status of the neighboring units of $i$.
    \item $\Delta_i$ is the interaction effect between unit $i$'s treatment status and the treatment of its neighbors.
\end{enumerate}

Given the treatment vector $\mbf{Z}_t$ and adjacency matrix $\mbf{A}_t$, the $ith$ row of the design matrix $H_i(\mbf{Z_t};\mbf{A}_t)$ is simply a list of indicators that select the appropriate parameters from $\boldsymbol{\theta}$ corresponding to the linear formulation of $r_i(\mbf{Z}_t;\mbf{A}_t)$. Thus, the design matrix $\mbf{H}$ can be constructed such that $\mbf{r}_t(\mbf{Z}_t;\mbf{A}_t) = \mbf{H}(\mbf{Z}_t;\mbf{A}_t)\boldsymbol{\theta} + \boldsymbol{\epsilon_t}$.

\section{Missing Proofs}\label{proofs}

\subsection{Proof of Theorem 1}
We prove Theorem 1 by first proving frequentist regret bounds for a linUCB style algorithm. 

\subsubsection{NetworkUCL Algorithm}
We adapt the linUCB algorithm from \citet{imp_lin_algs} to our setting. We prove novel regret bounds by an argument similar to \citet{imp_lin_algs}, adapting where necessary. We call this UCL-style algorithm networkUCL. It works by constructing confidence sets $C_t$ for $\boldsymbol{\theta}$ using data observed up to time $t$. It then maximizes the rewards jointly over the decision set $\{0,1\}^n$ and the confidence set. We provide basic pseudocode in Algorithm \ref{alg2}.

\begin{algorithm}[h]              
\caption{networkUCL algorithm.}  
\begin{algorithmic}[1]            

\For{ $t=1$ to $T$}
\State $(\mbf{Z}_t, \tilde{\boldsymbol{\theta}}) = \text{argmax}_{\mbf{Z} \in \{0,1\}^n,\ \boldsymbol{\theta}\in C_{t-1}}\mbf{H}(\mbf{Z}_t;\mbf{A}_t)\boldsymbol{\theta} $
\State Play $\mbf{Z}_t$, observe $\mbf{r}_t$
\State Update $C_t$
\EndFor

\end{algorithmic}\label{alg2}
\end{algorithm}

\subsubsection{NetworkUCL Regret Bound}
 We first require a bound on $\|\mathbf{H_t^\top\boldsymbol{\epsilon}\|_{\overline{V}_t^{-1}}}$, where $\ol{\mbf{V}}_t = \lambda I + \sum_{s=1}^t \mathbf{H}_s^\top \mathbf{H}_s$, and $\boldsymbol{\epsilon}$ is a vector of independent sub-gaussian random variables.
We follow a similar approach to Lemma 8, Lemma 9, and Theorem 1 of \citet{imp_lin_algs}. However, these proofs work with the martingale 
\[M_t^\lambda = \exp \left(\sum_{s=1}^t \left[ \frac{\epsilon_s\langle \lambda, H_s\rangle }{R} - \frac{1}{2}\langle \lambda, H_s \rangle^2  \right]  \right).\]  Our setting requires working with a  different martingale:
\[M_t^\lambda = \exp \left(\sum_{s=1}^t \left[ \sum_{i=1}^n\frac{\epsilon_{s,i}\langle \lambda, [H_s]_i\rangle }{R} - \frac{1}{2}\langle \lambda, [H_s]_i \rangle^2  \right]  \right).\] 

Throughout, we take $\{\mathcal{F}_t\}_{t=0}^\infty$ to be a filtration of $\sigma$-algebras such that $\mbf{H}_t$ is $\mathcal{F}_{t-1}$ measurable and $\epsilon_t$ is $\mathcal{F}_t$ measurable. We begin by adapting Lemma 8 of \citet{imp_lin_algs} to our model.

\begin{lemma}
Let $\lambda \in \mathbb{R}^d$ be arbitrary and for any $t \geq 0$ consider 
\[M_t^\lambda = \exp \left(\sum_{s=1}^t \left[ \sum_{i=1}^n\frac{\epsilon_{s,i}\langle \lambda, [X_s]_i\rangle }{R} - \frac{1}{2}\langle \lambda, [X_s]_i \rangle^2  \right]  \right).\] Let $\tau$ be a stopping time with respect to the filtration $\{\mathcal{F}\}_{t=0}^\infty$. Then under Assumption 4, $M_\tau^\lambda$ is almost surely well-defined and $\mathbb{E}[M_\tau^\lambda]\leq 1$.
\end{lemma}

\begin{proof}
To prove that $M_t^\lambda$ is a supermartingale, we must show that $\mathbb{E}[M_t^\lambda | \mathcal{F}_{t-1}]\leq M_{t-1}^\lambda$. We denote the overall contribution at time $t$ as $D_t^\lambda$ and the contribution of the $i$th row of $\mathbf{H}_t$ as $D_{t,i}^\lambda$:
    Let 
    \begin{align*}
    D_t^\lambda &= \exp \left(  \sum_{i=1}^n\frac{\epsilon_{t,i}\langle \lambda, [X_t]_i\rangle }{R} - \frac{1}{2}\langle \lambda, [X_t]_i \rangle^2  \right) \\
    D_{t,i}^\lambda &= \exp\left(\frac{\epsilon_{t,i}\langle \lambda, [X_t]_i\rangle }{R} - \frac{1}{2}\langle \lambda, [X_t]_i \rangle^2\right)
    \end{align*}
    Notice that $D_t^\lambda = D_{t,1}^\lambda \cdot D_{t,2}^\lambda\cdot\ldots\cdot D_{t,n}^\lambda. $ Now, by the 1-sub Gaussianity of $\epsilon_{t,i}$, we have that $\bb{E}[D_{t,i}^\lambda | \cc{F}_{t-1}] \leq 1$. Combining this with the independence of the $\epsilon$'s, we get
    \[\mathbb{E}[D_t^\lambda|\mathcal{F}_{t-1}] = \bb{E}[D_{t,1}^\lambda |\mathcal{F}_{t-1}] \cdots \bb{E}[D_{t,i}^\lambda |\mathcal{F}_{t-1}]\leq 1.\] We thus have 
    \begin{align*}
    \bb{E}\left[M_t^\lambda |\cc{F}_{t-1}\right] &= \bb{E}\left[D_1^\lambda\cdots D_t^\lambda | \cc{F}_{t-1}\right] \\
    &= D_1^\lambda \cdots D_{t-1}^\lambda \bb{E}\left[D_t^\lambda | \cc{F}_{t-1}\right] \\
    &= M_{t-1}^\lambda \bb{E}\left[D_t^\lambda | \cc{F}_{t-1}\right] \leq 1.
    \end{align*}
   The fact that $M_\tau^\lambda$ is well-defined follows from the proof of Lemma 8 of \citet{imp_lin_algs}.
\end{proof}

With this result in hand, Lemma 9 and Theorem 1 from \citet{imp_lin_algs} follow directly. We next need to extend Lemma 11 from \citet{imp_lin_algs} to our setting. Our proof differs due to the batch structure of our updates to $\ol{V}_{t}$, however our result is effectively equivalent to theirs.

\begin{lemma}[Adapted from Lemma 11 of \cite{imp_lin_algs}]
\label{lem:adapted_lemma_11}
Assume that for all $t$, the spectral norm of the design matrix is bounded such that $\|\mbf{H}_t\|_2 \le c_2$. Under the condition that the regularization parameter $\lambda \ge c_2^2$, it holds that
\[
\sum_{t=1}^T \sum_{i=1}^n \|[\mbf{H}_t]_i\|^2_{\ol{\mbf{V}}_{t-1}^{-1}} \le 2(\log\det\ol{\mbf{V}}_T - \log\det \mbf{V})
\]
where $\ol{\mbf{V}}_t = \mbf{V} + \sum_{s=1}^t \mbf{H}_s^\top \mbf{H}_s$ and $\mbf{V} = \lambda \mbf{I}$.
\end{lemma}

\begin{proof}
The proof adapts the core argument from \cite{imp_lin_algs} to our setting, where the update term $\mbf{H}_t^\top\mbf{H}_t$ may have a rank greater than one.

First, we express the inner sum over the $n$ nodes as a matrix trace. The term $\|[\mbf{H}_t]_i\|^2_{\ol{\mbf{V}}_{t-1}^{-1}}$ is a scalar and is thus equal to its own trace.
\begin{align*}
\sum_{i=1}^n \|[\mbf{H}_t]_i\|^2_{\ol{\mbf{V}}_{t-1}^{-1}} &= \sum_{i=1}^n \operatorname{tr}\left([\mbf{H}_t]_i \ol{\mbf{V}}_{t-1}^{-1} [\mbf{H}_t]_i^\top\right) \\
&= \sum_{i=1}^n \operatorname{tr}\left([\mbf{H}_t]_i^\top[\mbf{H}_t]_i \ol{\mbf{V}}_{t-1}^{-1}\right) && \text{by the cyclic property of the trace} \\
&= \operatorname{tr}\left(\left(\sum_{i=1}^n [\mbf{H}_t]_i^\top[\mbf{H}_t]_i\right) \ol{\mbf{V}}_{t-1}^{-1}\right) && \text{by the linearity of the trace} \\
&= \operatorname{tr}\left(\mbf{H}_t^\top\mbf{H}_t \ol{\mbf{V}}_{t-1}^{-1}\right)
\end{align*}
The quantity to bound is therefore $\sum_{t=1}^T \operatorname{tr}(\mbf{H}_t^\top\mbf{H}_t \ol{\mbf{V}}_{t-1}^{-1})$.

Next, we relate this trace to the growth of the log-determinant of $\ol{\mbf{V}}_t$. The determinant update rule is
\[
\det(\ol{\mbf{V}}_t) = \det(\ol{\mbf{V}}_{t-1} + \mbf{H}_t^\top\mbf{H}_t) = \det(\ol{\mbf{V}}_{t-1}) \cdot \det(\mbf{I} + \ol{\mbf{V}}_{t-1}^{-1} \mbf{H}_t^\top\mbf{H}_t).
\]
Taking the logarithm, we have
\[
\log\det(\ol{\mbf{V}}_t) - \log\det(\ol{\mbf{V}}_{t-1}) = \log\det(\mbf{I} + \mbf{A}_t),
\]
where $\mbf{A}_t = \ol{\mbf{V}}_{t-1}^{-1} \mbf{H}_t^\top\mbf{H}_t$. Note that $\operatorname{tr}(\mbf{A}_t) = \operatorname{tr}(\ol{\mbf{V}}_{t-1}^{-1} \mbf{H}_t^\top\mbf{H}_t) = \operatorname{tr}(\mbf{H}_t^\top\mbf{H}_t \ol{\mbf{V}}_{t-1}^{-1})$.

We use the inequality $x \le 2\log(1+x)$, which holds for $x \in [0, 1]$. This can be extended to matrices: for a positive semi-definite matrix $\mbf{A}_t$ with all its eigenvalues in $[0,1]$, we have $\operatorname{tr}(\mbf{A}_t) \le 2\log\det(\mbf{I}+\mbf{A}_t)$. The condition on the eigenvalues of $\mbf{A}_t$ is met if its largest eigenvalue, $\lambda_{\max}(\mbf{A}_t)$, is at most 1. We bound this eigenvalue as follows:
\[
\lambda_{\max}(\mbf{A}_t) \le \|\mbf{A}_t\|_2 \le \|\ol{\mbf{V}}_{t-1}^{-1}\|_2 \|\mbf{H}_t^\top\mbf{H}_t\|_2 = \lambda_{\max}(\ol{\mbf{V}}_{t-1}^{-1}) \|\mbf{H}_t\|_2^2.
\]
Since $\lambda_{\max}(\ol{\mbf{V}}_{t-1}^{-1}) = 1/\lambda_{\min}(\ol{\mbf{V}}_{t-1}) \le 1/\lambda$ and $\|\mbf{H}_t\|_2 \le c_2$ by assumption, we have $\lambda_{\max}(\mbf{A}_t) \le c_2^2 / \lambda$. Our stated condition $\lambda \ge c_2^2$ ensures that $\lambda_{\max}(\mbf{A}_t) \le 1$. Note that a similar assumption is made in \cite{imp_lin_algs}.

Thus, the inequality holds, and we have
\[
\operatorname{tr}(\mbf{H}_t^\top\mbf{H}_t \ol{\mbf{V}}_{t-1}^{-1}) \le 2 \log\det(\mbf{I} + \mbf{A}_t) = 2\left(\log\det(\ol{\mbf{V}}_t) - \log\det(\ol{\mbf{V}}_{t-1})\right).
\]
Summing this inequality from $t=1$ to $T$ yields a telescoping series:
\begin{align*}
\sum_{t=1}^T \operatorname{tr}(\mbf{H}_t^\top\mbf{H}_t \ol{\mbf{V}}_{t-1}^{-1}) &\le 2 \sum_{t=1}^T \left(\log\det(\ol{\mbf{V}}_t) - \log\det(\ol{\mbf{V}}_{t-1})\right) \\
&= 2\left(\log\det(\ol{\mbf{V}}_T) - \log\det(\mbf{V})\right).
\end{align*}
This completes the proof.
\end{proof}

We now prove an upper bound on frequentist regret for our UCL algorithm.
\setcounter{theorem}{2}
\begin{theorem}
    Under the conditions of Lemma 1 and Lemma 2 above as well as Theorem 2 from \cite{imp_lin_algs}, then with probability $1-\delta$ the networkUCL algorithm satisfies
    \[R_t\leq \sqrt{nTD\log (\lambda + nTL/D)}(\lambda^{1/2}S + R\sqrt{2\log(1/\delta) + D\log (1 + nTL/(\lambda D)} \]
\end{theorem}
\begin{proof}
    First, define $r_{t,i}$ as the regret accumulated by node $i$ at time $t$. The following inequality comes directly from the proof of Theorem 3 from \cite{imp_lin_algs}. Following their notation, denote $[H_t]_i^*$ as the optimal feature vector for node $i$ at time $t$. Additionally, $\theta^*$ denotes the true parameter vector, $\tilde{\theta}_t$ the optimistic choice, and $\hat{\theta}_t$ the OLS estimate.  
    \begin{align}
        r_{t,i} &= \langle [H_t]^*_i, \theta^*\rangle - \langle[H_t]_i, \theta^*\rangle\notag \\
        &\leq \langle [H_t]_i, \tilde{\theta}_t\rangle - \langle\theta^*, [H_t]_i\rangle\notag \\
        &= \langle[H_t]_i, \tilde{\theta}_t - \theta^*\rangle \notag\\
        &= \langle [H_t]_i, \hat{\theta}_{t-1} - \theta^*\rangle + \langle \tilde{\theta}_t - \hat{\theta}_{t-1}, [H_t]_i\rangle \notag\\
        &\leq \|\hat{\theta}_{t-1} - \theta^*\|_{\ol{V}_{t-1}} \|[H_t]_i\|_{\ol{V}_{t-1}^{-1}} + \|\tilde{\theta}_t - \hat{\theta}\|_{\ol{V}_{t-1}} \|[H_t]_i\|_{\ol{V}_{t-1}^{-1}}\notag \\
        &\leq 2 \sqrt{\beta_{t-1}(\delta)} \|[H_t]_i\|_{\ol{V}_{t-1}^{-1}} 
    \end{align}

    Using our previous lemma, we can now bound total cumulative regret. A simple extension of Lemma 10 from \cite{imp_lin_algs} gives us that $\log\det (\ol{V}_T) \leq D\log(\lambda + nTL/D)$, and Theorem 2 directly gives us that $\beta_{t}(\delta) \leq R\sqrt{D\log((1+nTL/\lambda)/\delta)} + \lambda^{1/2}S.$
    \begin{align*}
        R_t &= \sum_{t=1}^T \sum_{i=1}^n r_{t,i} = \sum_{t=1}^T \mbf{r}_t^\top\boldsymbol{1}_n \\
        &\leq \sum_{t=1} \sqrt{\|\mbf{r}_t\|_2^2 \|\boldsymbol{1}_n\|_2^2} = \sqrt{n} \sum_{t=1}^T \|\mbf{r}_t\|_2 \\
        &\leq \sqrt{nT\sum_{t=1}^T \|\mbf{r_t}\|_2^2} \\
        &= \sqrt{nT\sum_{t=1}^T\sum_{i=1}^n r_{t,i}^2}\\
        &\leq \sqrt{4nT\beta_{t-1}(\delta)\sum_{t=1}^T\sum_{i=1}^n \|[H_t]_i\|_{\ol{V}_{t-1}^{-1}}} && \text{by (1)} \\
        &\leq \sqrt{8nT\beta_{t-1}(\delta)\log\det (\ol{V}_T)} && \text{by Lemma 2} \\
        &\leq \sqrt{8nT\beta_{t-1}(\delta)D\log(\lambda + nTL/D)}\\
        &\leq \sqrt{nTD\log (\lambda + nTL/D)}(\lambda^{1/2}S + R\sqrt{2\log(1/\delta) + D\log (1 + nTL/(\lambda D)}
    \end{align*}
\end{proof}

\subsection{Proof of Theorem 1}

With Theorem 3, our result now follows directly from \cite{russo2014learning}.

\begin{proof}
  Given our frequentist regret bound, Theorem $1$ follows directly from the proof of Proposition $3$ in \cite{russo2014learning}.

  Specifically, Proposition $1$ of \cite{russo2014learning} (combined with the regret decomposition of the Bayesian regret of a UCB algorithm) states that the Bayesian regret of Thompson sampling is less than or equal to the Bayesian regret of a UCB-style algorithm. Because our frequentist regret bounds are for a UBC-style algorithm, we can directly use this result to bound the Bayesian regret of Thompson sampling by integrating our frequentist bounds over the distribution of the unknown parameter vector. Such an integral results in bounds of the same asymptotic order, thus giving us the bound in Theorem 1. 
\end{proof}

\subsection{Proof of Theorem 2}

\begin{proof}
   The proof proceeds by first constructing a specific hard problem instance and then showing any algorithm must incur a certain amount of regret on that problem.

\subsubsection*{Step 1: Hard Instance Construction}
Denote the set of possible parameter vectors as $\Theta = \{-\delta, \delta\}^D$ for some $\delta > 0$ to be chosen later. We consider a simple network of $n$ disconnected nodes. We assign to each node $j$ a fixed feature vector $w_j \in \{-1, 1\}^D$. These vectors are constructed from the rows of a $D \times D$ Hadamard matrix. The reward for a node $j$ is linear in its features: $r_{t,j} = Z_{t,j} \langle w_j, \theta \rangle + \epsilon_{t,j}$, where $\epsilon_{t,j} \sim \mathcal{N}(0,1)$. The design matrix $H(Z_t)$ is formed by stacking the row vectors $w_j$ for all treated nodes $j$ (where $Z_{t,j}=1$).

To analyze the regret, we focus on a specific binary choice for each dimension $i \in \{1, \dots, D\}$. We partition the nodes into two sets: $G_{i,A} = \{j \mid w_{ji} = +1\}$ and $G_{i,B} = \{j \mid w_{ji} = -1\}$. By the properties of Hadamard matrices, $|G_{i,A}| = |G_{i,B}| = n/2$. We define two actions:
\begin{itemize}
    \item \textbf{Action $Z_{i,A}$}: Treat all nodes in group $G_{i,A}$.
    \item \textbf{Action $Z_{i,B}$}: Treat all nodes in group $G_{i,B}$.
\end{itemize}
If the true parameter is $\theta = \delta \cdot u_i$ (where $u_i$ is the standard basis vector), the total expected reward for action $Z_{i,A}$ is $\sum_{j \in G_{i,A}} \langle w_j, \delta u_i \rangle = (n/2) \cdot \delta$. The total expected reward for $Z_{i,B}$ is $\sum_{j \in G_{i,B}} \langle w_j, \delta u_i \rangle = (n/2) \cdot (-\delta)$. The regret from making a single error by choosing the wrong action is therefore $n\delta$.

\subsubsection*{Step 2: Information-Theoretic Bound}
We use the Bretagnolle-Huber inequality to lower bound the probability of error. Let's consider distinguishing between two parameter vectors $\theta$ and $\theta'$, where $\theta$ and $\theta'$ differ only on the $i$-th component, i.e., $\theta_i' = -\theta_i$. Let $p_{\theta i}$ be the probability of the ``bad event" where an algorithm chooses the suboptimal action corresponding to dimension $i$ for at least half of the $T$ rounds. The inequality states
$$p_{\theta i} + p_{\theta' i} \geq \frac{1}{2}\exp\left(-\text{KL}(\mathbb{P_\theta} ||\mathbb{P}_{\theta'})\right).$$
The KL divergence is given by
\begin{align*}
    \text{KL}(\mathbb{P}_\theta||\mathbb{P}_{\theta'}) &= \frac{1}{2} \sum_{t=1}^T\mathbb{E}\left[\|H(Z_t)(\theta - \theta')\|_2^2 \right] = 2\delta^2 \sum_{t=1}^T \mathbb{E}\left[\| [H(Z_t)]_i \|_2^2\right]
\end{align*}
In our construction, if the algorithm chooses $Z_{i,A}$ or $Z_{i,B}$, it treats $n/2$ nodes. The $i$-th column of the design matrix, $[H(Z_t)]_i$, will have $n/2$ entries of $\pm 1$. Its squared L2-norm is $\|[H(Z_t)]_i\|_2^2 = n/2$. The KL divergence is thus bounded by
$$ \text{KL}(\mathbb{P}_\theta||\mathbb{P}_{\theta'}) \le 2\delta^2 \sum_{t=1}^T \frac{n}{2} = \delta^2 nT $$
To ensure the KL divergence is a constant (e.g., 1), we choose $\delta = \frac{1}{\sqrt{nT}}$.
Then, using the ``averaging hammer" argument \citep{Lattimore_Szepesvári_2020} over all pairs $(\theta, \theta')$ and all dimensions $i$, we can show that there exists a ``hard" parameter vector $\theta^*$ for which $\sum_{i=1}^D p_{\theta^*i} \ge \frac{CD}{2}$ for some constant $C$.

\subsubsection*{Step 3: Regret Analysis}
We now connect the probability of error to the total regret. For the constructed hard instance, the regret incurred from a single error on dimension $i$ is $n\delta$.
\begin{align*}
    R_T(\theta^*) &= \mathbb{E}\left[\sum_{t=1}^T \text{regret}_t\right] \\
    &\ge \sum_{i=1}^D \mathbb{E}\left[ \sum_{t=1}^T n\delta \cdot \boldsymbol{1}\{\text{error on component } i \text{ at time } t \}\right] \\
    &= n\delta \sum_{i=1}^D \mathbb{E}\left[\sum_{t=1}^T\boldsymbol{1}\{\text{error on } i \text{ at time } t\}  \right] \\
    &\geq n\delta \frac{T}{2} \sum_{i=1}^D \mathbb{P}\left(\sum_{t=1}^T\boldsymbol{1}\{\text{error on } i\} \geq \frac{T}{2}\right) && \text{(by Markov's Inequality)} \\
    &= \frac{n\delta T}{2} \sum_{i=1}^D p_{\theta^*i} \\
    &\geq \frac{n\delta T}{2} \left(\frac{CD}{2}\right) && \text{(from the Averaging Hammer)} \\
    &= \frac{CnDT\delta}{4} \\
    &= \frac{CnDT}{4}\frac{1}{\sqrt{nT}} = \frac{C}{4}D\sqrt{nT}
\end{align*}
Thus, we have shown that $R_T(\mathcal{A}, \theta^*) \ge \Omega(D\sqrt{nT})$.
\end{proof}

\section{Additional Experiments}

\subsection{Effect of the Regularization Parameter}\label{appendix:lambda}
The choice of $\lambda$ does not impact our regret bounds as its effect diminishes asymptotically with $T\to \infty$. This can be seen through the posterior update formulas for the normal-normal Bayesian linear regression model. The posterior precision and mean are given by 
\begin{align*}
    \Sigma_T^{-1} &= \lambda \mbf{I}_D + \frac{1}{\sigma^2} \sum_{t=1}^T \mbf{H}_t^T \mbf{H}_t \\
    \mu_T &= \left(\lambda \mbf{I}_D + \frac{1}{\sigma^2} \sum_{t=1}^T \mbf{H}_t^T \mbf{H}_t\right)^{-1} \left(\frac{1}{\sigma^2} \sum_{t=1}^T \mbf{H}_t^T \mbf{r_t}\right)
\end{align*}

As the amount of data increases, $\sum_{t=1}^T \mbf{H}_t^\top \mbf{H}_t$ ``swamps out'' the constant term $\lambda \mbf{I}_d$ \citep{Vaart_1998}.

However, the regularization parameter can have a significant impact on regret in the early stages of the algorithm. In Figure \ref{fig:lambda}, we show the impact of $\lambda$ on regret for the grouped model of Section 3.2.2. A large $\lambda$ slows learning as the data struggles to overcome the weight of the prior, while a small $\lambda$ can cause the algorithm to overfit initially, causing increased regret. We see in Figure \ref{fig:lambda} that $\lambda=1$ appears optimal. 

\begin{figure}[htbp]
    \centering
    \includegraphics[width=0.5\linewidth]{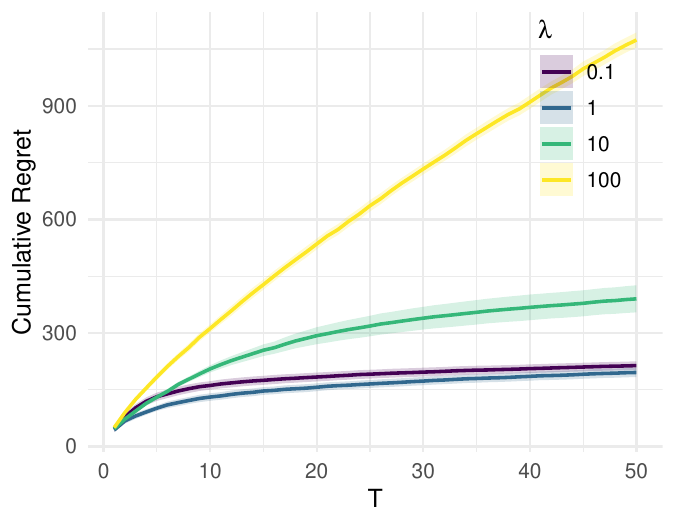}
    \caption{Impact of $\lambda$ on regret for the grouped effects model.}
    \label{fig:lambda}
\end{figure}

\section{Computing Resources}
 Figures were produced on a high-performance computing cluster using 50-75 CPUs in parallel, each assigned 1GB of RAM. The CPUs used were Intel(R) Xeon(R) Gold 6336Y. Gurobi software was used for all simulations. The authors received an academic license for free by registering through the Gurobi website. 

\end{document}